\documentclass[letterpaper, 10 pt, journal, twoside]{./IEEEtran}

\usepackage{graphicx}
\usepackage{subcaption}
\usepackage{amsmath}

\usepackage{amssymb}  
\usepackage{amsthm}
\usepackage{algorithm}
\usepackage{algorithmicx}
\usepackage{epstopdf}
\usepackage{hyperref}

\newtheorem{assumption}{Assumption}
\newtheorem{remark}{Remark}
\newtheorem{theorem}{Theorem}

\newtheorem{proposition}{Proposition}
\newtheorem{lemma}{Lemma}

\title{Safe and Fast Tracking on a Robot Manipulator: Robust MPC and Neural Network Control} 

\author{Julian Nubert$^{1}$, Johannes K\"ohler$^2$, Vincent Berenz$^3$, Frank Allg\"ower$^2$, and Sebastian Trimpe$^4$
\thanks{Manuscript accepted January, 14, 2020.}
\thanks{This letter was recommended for publication by Associate Editor S. Briot and Editor P. Rocco upon evaluation of the reviewers' comments. This work was supported in part by the Max Planck Society, by the Cyber Valley Initiative, and by the German Research Foundation (grant GRK 2198/1).}
\thanks{$^{1}$Julian Nubert is with the Intelligent Control Systems Group, Max Planck Institute for Intelligent Systems, 70569 Stuttgart, Germany, and also M.Sc. student, ETH Z\"urich, 8092 Z\"urich, Switzerland (nubertj@ethz.ch).}
\thanks{$^{2}$Johannes K\"ohler and Frank Allg\"ower are with the Institute for Systems Theory and Automatic Control, University of Stuttgart, 70550 Stuttgart, Germany (\{johannes.koehler, frank.allgower\}@ist.uni-stuttgart.de).}
\thanks{$^{3}$Vincent Berenz is with the Autonomous Motion Department at the Max Planck Institute for Intelligent Systems, 72076 T\"ubingen, Germany (vberenz@tuebingen.mpg.de).}
\thanks{$^{4}$Sebastian Trimpe is with the Intelligent Control Systems Group, Max Planck Institute for Intelligent Systems, 70569 Stuttgart, Germany (trimpe@is.mpg.de).}
\thanks{Digital Object Identifier (DOI): 10.1109/LRA.2020.2975727.}
}

\usepackage{fancyhdr}
\newcommand{\mytitle}{\textbf{Accepted version.} To appear in \textit{IEEE Robotics and Automation Letters}.  DOI:
10.1109/LRA.2020.2975727\\
\copyright 2020 IEEE. Personal use of this material is permitted.
Permission from IEEE must be obtained for all other uses, in any current or future media, including reprinting/republishing this material for advertising or promotional purposes, creating new collective works, for resale or redistribution to servers or lists, or reuse of any copyrighted component of this work in other works.} 
\begin{document}

\maketitle

\thispagestyle{fancy}
 
\begin{abstract}
Fast feedback control and safety guarantees are essential in modern robotics.  We present an approach that achieves both by combining novel robust model predictive control (MPC) with function approximation via (deep) neural networks (NNs).  The result is a new approach for complex tasks with nonlinear, uncertain, and constrained dynamics as are common in robotics.  Specifically, we leverage recent results in MPC research to propose a new robust setpoint tracking MPC algorithm, which achieves reliable and safe tracking of a dynamic setpoint while guaranteeing stability and constraint satisfaction. 
The presented robust MPC scheme constitutes a one-layer approach that unifies the often separated planning and control layers, by directly computing the control command based on a reference and possibly obstacle positions.
As a separate contribution, we show how the computation time of the MPC can be drastically reduced by approximating the MPC law with a NN controller. The NN is trained and validated from offline samples of the MPC, yielding statistical guarantees, and used in lieu thereof at run time.
Our experiments on a state-of-the-art robot manipulator are the first to show that both the proposed robust and approximate MPC schemes scale to real-world robotic systems.
\end{abstract}


\IEEEpeerreviewmaketitle


\section{Introduction}
\IEEEPARstart{T}{he} need to handle complexity becomes more prominent in modern control design, especially in robotics. 
First of all, 
complexity often stems from tasks or system descriptions that are high-dimensional and nonlinear. 
Second, not only classic control properties, e.g. step-response characteristics, are of interest, but also additional guarantees such as stability or satisfaction of hard constraints on inputs and states. In particular, the ability to \emph{robustly} guarantee safety becomes absolutely essential when humans are involved within the process, such as for automated driving or human-robot interaction. 
Finally, many robotic systems and tasks require fast acting controllers in the range of milliseconds, which is exacerbated by the need to run algorithms on resource-limited hardware.

Designing controllers for such challenging applications often involves the combination of several different conceptual layers. For example, classical robot manipulator control involves trajectory planning in the task space, solving for the inverse kinematics of a single point (i.e., the setpoint) or multiple points (task space trajectory), and the determination of required control commands in the state space~\cite{Siciliano2008}. These approaches can be affected by corner cases  of one of the components; for example, solving for the inverse kinematics may not be trivial for redundant robots.
For many complex scenarios, a direct approach is hence desirable for tracking of (potentially unreachable) reference setpoints in task space.

\begin{figure}[t]
	\centering
	\includegraphics[width=1.0\linewidth]{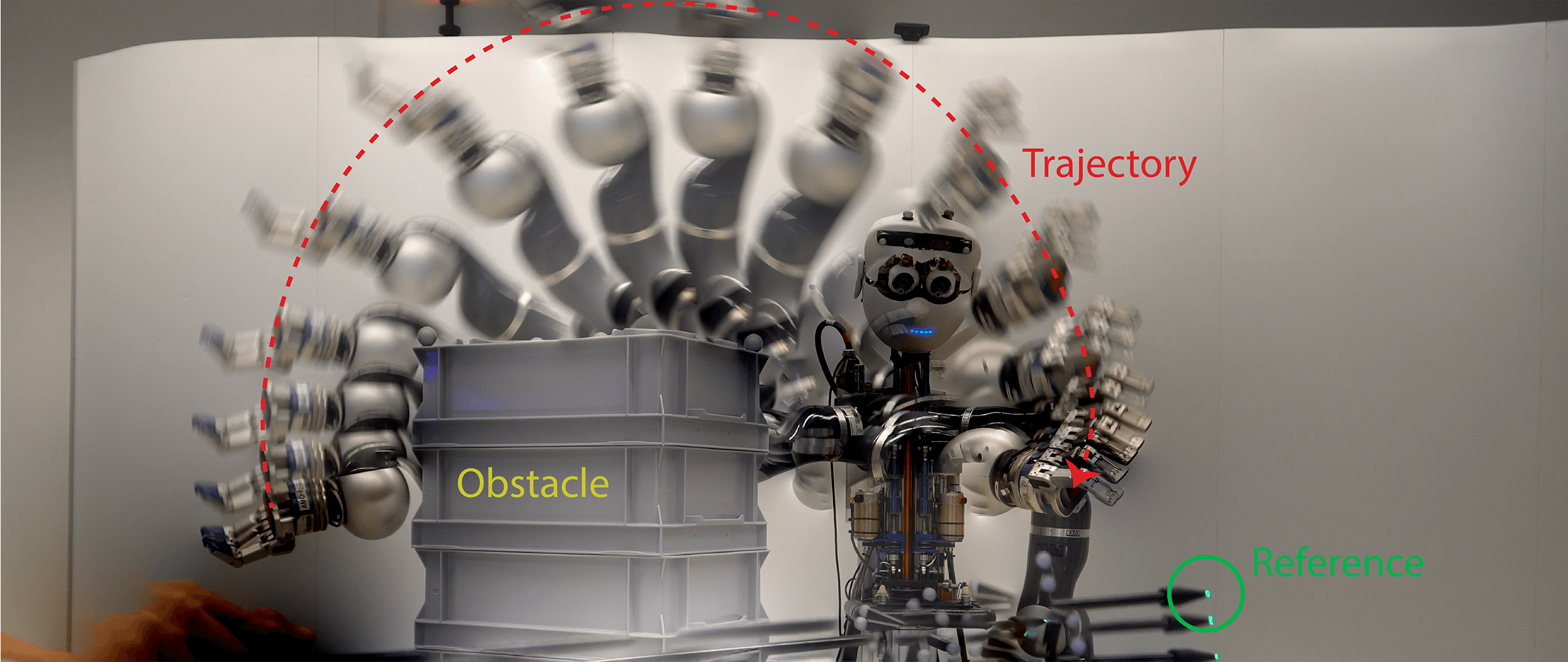}
	\caption{Apollo robot~\cite[Sec.~VI]{2017_rss_system} with two LBR4+ arms (at MPI-IS T\"ubingen). The end effector tracks the reference encircled in green, while guaranteeing stability and constraint satisfaction at all times (e.g., avoiding obstacles).}
	\label{fig:robot_experiment}
\end{figure}

In this letter, we propose a single-layer approach for robot tracking control that handles all aforementioned challenges.  We achieve this by combining (and extending) recent robust model predictive control (RMPC) and function approximation via supervised learning with (deep) neural networks (NNs).  The proposed RMPC can handle nonlinear systems, constraints, and uncertainty. 
In order to overcome the  computational complexity inherent in the online MPC optimization, we present a solution that approximates the RMPC with supervised learning yielding a NN as an explicit control law and a speed improvement by two orders of magnitude.
Through experiments on a KUKA LBR4+ robotic manipulator (see Figure~\ref{fig:robot_experiment}), we demonstrate -- for the first time -- the feasibility of both the novel robust MPC and its NN approximation for robot control.

\subsubsection*{Related Work}
MPC can handle nonlinear constraints and is applicable to nonlinear systems \cite{rawlings2009model}, however, disturbances or uncertainty can compromise the safety guarantees of nominal MPC schemes. RMPC overcomes this by preserving safety and stability despite disturbances and uncertainty.

Tube-based RMPC schemes ensure constraint satisfaction by predicting a tube around the nominal (predicted) trajectory that confines the actual (uncertain) system trajectory.
In~\cite{villanueva2017robust}, an approach based on min-max differential inequalities is presented to achieve robustness for the nonlinear case. 
In this work, we build upon the novel \emph{nonlinear} constraint tightening approach in~\cite{KoehlerCompEff18}, which provides slightly more conservative results than the approach in~\cite{villanueva2017robust}, but is far more computationally efficient.

We herein extend \cite{KoehlerCompEff18} to setpoint tracking. 
Setpoint tracking MPC, as introduced in~\cite{LIMON18} for nonlinear systems, enables the controller to track piece-wise constant output reference signals. 
A robust version for linear systems is presented in~\cite{limon2010robust}. 
To obtain a robust version for nonlinear systems, we optimize the size of the terminal set around the artificial steady state online, similar as done in~\cite{kohler19nonlinear} for nominal MPC.
None of the aforementioned robust or setpoint tracking MPC approaches has been applied before on a real-world, safety-critical system of complexity similar to the robot arm herein. 

Approximate MPC (AMPC) allows for running high performance control on relatively cheap hardware by using supervised learning (e.g. NNs) to approximate the implicit solution of the optimization problem.
Recently, in~\cite{DeepMPCChen18,zhang2019near,karg18},~\cite{chen2019large} theoretical approaches for AMPC for linear systems were presented, which use projection/active set iterations for feasibility~\cite{DeepMPCChen18,chen2019large}, statistical validation~\cite{karg18}, and duality for performance bounds~\cite{zhang2019near}. 
Herein, we leverage the AMPC approach for \emph{nonlinear} systems, recently proposed in~\cite{hertneck18}, which yields a NN control law that inherits the MPC's guarantees (in a statistical sense) through robust design and statistical validation. 

MPC control for robotic manipulators is investigated, for example, in~\cite{Faulwasser17,carron2019data}.
However, both of these approaches assume a trajectory in the joint space to be given beforehand. In~\cite{PredInvKin19}, reference tracking in the task space by using task scaling to solve for the inverse kinematics of a redundant manipulator is proposed, taking kinematic limits into account. In none of these approaches, safety guarantees or robustness under uncertainty are considered. 
Approaches making use of robust MPC schemes are not widely used in robotics (yet),
but tube and funnel approaches have  recently been explored for robust robot motion planning~\cite{majumdar2017funnel,singh2017robust,fridovich2018planning}.
However, to the best of our knowledge, no experimental implementation of an MPC design with theoretically guaranteed robustness exists yet for a robotic system. 

\subsubsection*{Contributions}
This letter makes contributions in three main directions: \emph{(i)} robust setpoint tracking MPC, \emph{(ii)} approximate MPC via supervised learning with NNs, and \emph{(iii)} their application to real robotic systems. 
\emph{(i)} We present a new RMPC setpoint tracking approach that combines the RMPC~\cite{KoehlerCompEff18} with the MPC setpoint tracking in~\cite{LIMON18} by proposing online optimized terminal ingredients to improve performance subject to safety constraints. 
The resulting robust approach provides safety guarantees in face of disturbances and uncertainties while  yielding fully integrated robot control in one-layer (i.e., robust motion planning and feedback control). 
\emph{(ii)} The presented AMPC builds and improves upon the approach in~\cite{hertneck18} by providing a novel, less conservative validation criterion that also considers model mismatch, which is crucial for robot experiments.
The proposed AMPC considerably improves performance due to fast NN evaluation, while providing statistical guarantees on safety. 
\emph{(iii)} Finally, this work comprises the first experimental implementations of both, the RMPC based on~\cite{KoehlerCompEff18} and the AMPC originating from~\cite{hertneck18}.
To the best of our knowledge, this is the first experimental implementation of nonlinear tracking RMPC with safety properties theoretically guaranteed by design.


\section{Problem Formulation}
\label{sec:general_approach}
We consider disturbed nonlinear continuous-time systems
\begin{equation}
\label{equ:system_ct}
    \dot{x}_t=f_{\mathrm{c}}(x_t,u_t)+d_{\mathrm{w},\mathrm{c},t},~y_t=o(x_t,u_t),
\end{equation}
with state $x_t \in \mathbb{R}^n$, control input $u_t \in \mathbb{R}^m$, output $y_t \in \mathbb{R}^q$, nominal dynamics $f_\mathrm{c}$ and model mismatch $d_{\mathrm{w},\mathrm{c},t} \in \mathcal{W}(x_t,u_t)$ with some known compact set $\mathcal{W}$. For the nonlinear state, input and output constraint set $\mathcal{Z}$, we consider
\begin{equation}
\label{equ:constraint_formulation}
    \mathcal{Z} = \{(x,u) \in \mathbb{R}^{n+m} | \bar{g}_j(x,u,o(x,u)) \leq 0,~j = 1,\dots,p\}. \nonumber
\end{equation}
In the following, we denote $g_j(x,u):=\bar{g}_j(x,u,o(x,u))$ and omit the time index $t$ when clear from context.
 
\subsubsection*{Objective}
Given an output reference $y^{\mathrm{d}}_t$, the control goal is to exponentially stabilize the optimal reachable setpoint, while ensuring robust constraint satisfaction, i.e. $(x_t,u_t) \in \mathcal{Z}~\forall t\geq 0$. 
This should hold irrespective of the reference, and even for a non-reachable output reference $y^\mathrm{d}$. To meet requirements of modern robotics, the controller should operate at fast update rates, e.g., ideally at the order of milliseconds. 

Such control problems are ubiquitous in robotics and other areas and combine the challenges of safe and fast tracking for complex (i.e., nonlinear, uncertain, constrained) systems.


\section{Methods: RMPC Setpoint Tracking \& AMPC}
\label{sec:rmpc_design}

In this section, we introduce the RMPC scheme based on~\cite{KoehlerCompEff18} (Sec.~\ref{sec:Method_RMPC}) and extend it to robust output tracking (Sec.~\ref{sec:Method_Tracking}). Following this, we show how the online control can be accelerated by moving the optimization offline using AMPC (Sec.~\ref{sec:Method_AMPC}) as an extension to the approach in~\cite{hertneck18}.

\subsection{Robust MPC Design}
\label{sec:Method_RMPC}
To ensure fast feedback, the piece-wise constant MPC control input $\pi_{\text{MPC}}$ is combined with a continuous-time control law $\kappa(x)$, i.e. the closed-loop input is given by
\begin{equation}
\label{equ:controlled_system}
    u_t=\pi_{\text{MPC}}\left(x_{t_k},y^{\mathrm{d}}_{t_k}\right) + \kappa(x_t),~\forall t \in [t_k, t_k+h],
\end{equation}
where $h$ denotes the sampling time of the RMPC, $t_k=k h$ the sampling instance, and $\pi_{\text{MPC}}$ the piece-wise constant MPC control law.
Denote $f_{\mathrm{c},\kappa}(x,v)=f_{c}(x,v+\kappa(x))$, $g_{j,\kappa}(x,v)=g_j(x,v+\kappa(x))$, $o_\kappa(x,v)=o(x,v+\kappa(x))$ and $\mathcal{W}_\kappa(x,v)=\mathcal{W}(x,v+\kappa(x))$.

\subsubsection{Incremental Stability}
For the design of the RMPC, we assume that the feedback $\kappa$ ensures incremental exponential stability, similar to \cite[Ass.~9]{KoehlerCompEff18}.
\begin{assumption}
\label{ass:loc_inc_stab}
There exists an incremental Lyapunov function $V_{\delta}:\mathbb{R}^n\times\mathbb{R}^n\rightarrow\mathbb{R}_{\geq 0}$ and constants $c_{\delta,\mathrm{l}},c_{\delta,\mathrm{u}},c_j,\rho_{\mathrm{c}} > 0$ s.t. the following properties hold $\forall(z,v+\kappa(z))\in\mathcal{Z}$, $x\in\mathbb{R}^n$:
\begin{subequations}
\begin{align}
\label{equ:inc_stab_equ_1}
& c_{\delta,\mathrm{l}} ||x-z||^2 \leq V_{\delta}(x,z) \leq c_{\delta,\mathrm{u}} ||x-z||^2, \\
\label{equ:inc_stab_equ_2}
& g_{j,\kappa}(x,v)-g_{j,\kappa}(z,v) \leq c_j \sqrt{V_{\delta}(x,z)}, \\
\label{equ:inc_stab_equ_3}
& \frac{d}{dt}V_{\delta}(x,z) \leq -2\rho_{\mathrm{c}} V_{\delta}(x,z), 
\end{align}
\end{subequations} 
with $\dot{x} = f_{\mathrm{c},\kappa}(x,v)$, $\dot{z} = f_{\mathrm{c},\kappa}(z,v)$.
Furthermore, the following norm-like inequality holds $\forall x_1,x_2,x_3\in\mathbb{R}^n$:
\begin{equation}
\label{equ:norm_like_cond}
\sqrt{V_{\delta}(x_1,x_2)}+\sqrt{V_{\delta}(x_2,x_3)}\geq \sqrt{V_{\delta}(x_1,x_3)}.
\end{equation}
\end{assumption}
The first and third condition (\eqref{equ:inc_stab_equ_1},~\eqref{equ:inc_stab_equ_3}) formulate stability while the second is fulfilled for locally Lipschitz continuous $g_{j,\kappa}$.
Incremental stability is a rather general condition, among others allowing for the usage of standard polytopic and ellipsoidal Lyapunov functions $V_{\delta}$ (i.e. $V_{\delta}(x,z)=\|x-z\|_{P_\delta}^2$), which satisfy condition~\eqref{equ:norm_like_cond} due to the triangular inequality. Compare~\cite[Remark~1]{KoehlerCompEff18} for a general discussion.

\subsubsection{Tube}
In this work, we use $V_\delta$ to characterize the tube around the nominal trajectory according to the system dynamics $\dot{z} = f_{\mathrm{c},\kappa}(z,v)$.
The predicted tube is parameterized by $\mathbb{X}_{\tau|t}=\{x|~V_{\delta}(x,z_{\tau|t})\leq s_{\tau|t}^2\}$, where $z_{\tau|t}$ denotes the nominal prediction and the tube size $s_{\tau|t}\geq 0$ is a scalar. For the construction of the tube and hence, for the design of the RMPC controller, we use a characterization of the magnitude of occurring uncertainties.

\subsubsection{Disturbance Description}

To over-approximate the uncertainties arising from the model mismatch $d_{\mathrm{w},\mathrm{c},t} \in \mathcal{W}(x_t,u_t)$, we need a (possibly constant) function $\overline{w}_{\mathrm{c}}$.
Using~\ref{equ:inc_stab_equ_3} and  $\mathcal{W}(x,u)$, it is possible to construct $\overline{w}_{\mathrm{c}}$ satisfying
\begin{align}
\label{eq:V_delta_w_diff}
    & \frac{d}{dt}\sqrt{V_\delta(x,z)} + \rho_{\mathrm{c}}\sqrt{V_\delta(x,z)} \leq \bar{w}_{\mathrm{c}}(z,v,\sqrt{V_{\delta}(z,v)}),\\
    & \dot{x}=f_{\mathrm{c},\kappa}(x,v)+d_{\mathrm{w},\mathrm{c}},~\dot{z}=f_{c,\kappa}(z,v),
    \forall~ d_{\mathrm{w},\mathrm{c}}\in\mathcal{W}_\kappa(x,v).\nonumber
\end{align}
The state and input dependency of $\bar{w}_\mathrm{c}$ can e.g. represent larger uncertainty in case of high dynamic operation due to parametric uncertainty.  For simplicity, we only consider a positive constant $\overline{w}_{\mathrm{c}}>0$ in the following, for details regarding the general case see~\cite{KoehlerCompEff18}.

\subsubsection{Tube Dynamics and Design Quantities}
\label{sec:design_quantities}
By using inequality~\eqref{eq:V_delta_w_diff}, the tube propagation is given by $\dot{s}_t = -\rho_{\mathrm{c}} s_t+\overline{w}_{\mathrm{c}}$, yielding $s_t=\overline{w}_{\mathrm{c}}/\rho_{\mathrm{c}}(1-e^{-\rho_{\mathrm{c}} t})$.
To allow for an efficient online optimization, we consider the discrete-time system $x^+=f_{\mathrm{d},\mathrm{w},\kappa}(x,v)=f_{\mathrm{d},\kappa}(x,v)+d_{\mathrm{w},\mathrm{d}}$, where $f_{\mathrm{d},\kappa}$ is the discretization of $f_{\mathrm{c},\kappa}$ with sampling time $h$ and $d_{\mathrm{w},\mathrm{d}}\in\mathcal{W}_d(x,v)$ denoting the discrete-time model mismatch.
Given the sampling time $h$, the corresponding discrete-time tube size is given by $s_{k\cdot h}=\frac{1-\rho_{\mathrm{d}}^k}{1-\rho_{\mathrm{d}}}\overline{w}_\mathrm{d}$ with $\rho_\mathrm{d}=e^{-\rho_{\mathrm{c}} h}$, $\overline{w}_{\mathrm{d}}=s_h$.
The discrete-time model mismatch satisfies $\sqrt{V_{\delta}(f_{\mathrm{d},\kappa}(x,v)+d_{\mathrm{w},\mathrm{d}},f_{\mathrm{d},\kappa}(x,v))}\leq \overline{w}_{\mathrm{d}}$, $\forall d_{\mathrm{w},\mathrm{d}}\in\mathcal{W}_{\mathrm{d}}(x,v)$.
The contraction rate $\rho_\mathrm{d}$ defines the growing speed of the tube while $s_{k \cdot h}$ denotes the size of the tube around the nominal trajectory, which bounds the uncertainties.

\subsection{Robust Setpoint Tracking}
\label{sec:Method_Tracking}
A standard MPC design (c.f.~\cite{rawlings2009model}) minimizes the squared distance $\|x(k|t)-x^{\mathrm{d}}\|_Q^2 + \|u(k|t)-u^d\|_R^2$ to some desired setpoint $(x^\mathrm{d},u^\mathrm{d})$, which requires a feasible target reference in the state and input space.
For the considered problem of (robust) setpoint tracking of the output $y$ (the end effector position in Sec.~\ref{sec:experiment}), this would require a (usually unknown) mapping of the form
$x^\mathrm{d} = m_x(y^\mathrm{d}),~u^\mathrm{d} = m_u(y^\mathrm{d})$.
\begin{remark}
In our specific use case of controlling a robotic manipulator, $m_x$ corresponds to the inverse kinematics. For MPC-based robot control such mappings are used in~\cite{Faulwasser17,carron2019data}, which we particularly avoid within our work.
\end{remark}

The proposed approach is a combination of~\cite{KoehlerCompEff18} and~\cite{LIMON18} and hence, can be seen as an extension of~\cite{kohler19nonlinear} to the robust case.
The following optimization problem characterizes the proposed RMPC scheme for setpoint tracking and avoids the need of providing $m_x,m_u$:
\begin{subequations}
\label{equ:rmpc_opt_problem}
\begin{alignat}{2}
& V_N(x_{t},y^{\mathrm{d}}_{t}) & & = \!\min_{v(\cdot|t),x^{\mathrm{s}},v^{\mathrm{s}},\alpha} J_N(x_{t},y^{\mathrm{d}}_{t};v(\cdot|t),y^{\mathrm{s}},x^{\mathrm{s}},v^{\mathrm{s}}) \nonumber \\
& \text{subject to} & & x(0|t) = x_t, \nonumber \\
\label{equ:dyn_pred}
& & & x(k+1|t) = f_{\mathrm{d},\kappa}(x(k|t), v(k|t)), \\
\label{equ:tightened_const}
& & & g_{j,\kappa}(x(k|t),v(k|t)) + c_j \frac{1-\rho_\mathrm{d}^k}{1-\rho_\mathrm{d}}\overline{w}_\mathrm{d} \leq 0, \\
\label{equ:new_constraints}
& & & x^s = f_{\mathrm{d},\kappa}(x^\mathrm{s},v^\mathrm{s}), \quad y^\mathrm{s} = o_\kappa(x^\mathrm{s},v^\mathrm{s}), \\
\label{equ:constr_term_set_1}
& & & \frac{\overline{w}_\mathrm{d}}{1-\rho_{\mathrm{d}}} \leq \alpha \leq -\frac{g_{j,\kappa}(x^\mathrm{s},v^\mathrm{s})}{c_j}, \\
\label{equ:constr_term_set_2}
& & & x(N|t) \in \mathcal{X}_{\mathrm{f}}(x^\mathrm{s},\alpha), \\ 
& & & k = 0,...,N-1, \quad j = 1,...,p, \nonumber
\end{alignat}
\end{subequations}
with the objective function
\begin{equation}
\label{equ:objective_func}
\!\begin{aligned}
    & J_N(x_{t},y^\mathrm{d}_{t};v(\cdot|t),y^\mathrm{s},x^\mathrm{s},v^\mathrm{s}) \\
    := & \sum_{k=0}^{N-1} \left(||x(k|t)-x^{\mathrm{s}}||_Q^2 + ||v(k|t)-v^{\mathrm{s}}||_R^2\right) \\ 
    & + V_{\mathrm{f}}(x(N|t),x^{\mathrm{s}}) + ||y^{\mathrm{s}}-y^{\mathrm{d}}_t||_{Q_{\mathrm{o}}}^2,
\end{aligned}
\end{equation}
$Q,R,Q_{\mathrm{o}} \succ 0$.
The terminal set is given as
\begin{equation}
\label{equ:constr_term_set_3}
     \mathcal{X}_{\mathrm{f}}(x^{\mathrm{s}},\alpha) := \{x \in \mathbb{R}^n | \sqrt{V_\delta(x,x^{\mathrm{s}})} + \frac{1-\rho_{\mathrm{d}}^N}{1-\rho_{\mathrm{d}}}\overline{w}_{\mathrm{d}} \leq \alpha \}.
\end{equation}
The optimization problem~\eqref{equ:rmpc_opt_problem} is solved at time $t$ with the initial state $x_t$. 
The optimal input sequence is denoted by $v^*(\cdot|t)$ with the control law denoted as $\pi_{\text{MPC}}(x_t,y_t^{\mathrm{d}})=v^*(0|t)$.
The predictions along the horizon $N$ are done w.r.t. the nominal system description in~\eqref{equ:dyn_pred}.
Furthermore, the constraints in~\eqref{equ:tightened_const} are tightened with tube size $s_{k \cdot h}$. 
In the following, we explain the considered objective function $J_N$ in~\eqref{equ:objective_func} and the conditions for the terminal set $\mathcal{X}_{\mathrm{f}}$ in~\eqref{equ:constr_term_set_1},~\eqref{equ:constr_term_set_2} and~\eqref{equ:constr_term_set_3} for setpoint tracking in more detail.

\subsubsection{Objective Function}

To track the external output reference $y^{\mathrm{d}}$, we use the setpoint tracking formulation introduced by Limon et al.~\cite{LIMON18}. 
Additional decision variables $(x^{\mathrm{s}},v^{\mathrm{s}})$ are used to define an artificial steady-state \eqref{equ:new_constraints}. 
The first part of the objective function $J_N$ ensures that the MPC steers the system to the artificial steady-state, while the term $\|y^{\mathrm{s}}-y^{\mathrm{d}}_t\|_{Q_{\mathrm{o}}}^2$ ensures that the output $y^{\mathrm{s}}$ at the artificial steady-state $(x^{\mathrm{s}},v^{\mathrm{s}})$ tracks the desired output $y^{\mathrm{d}}$.
In Theorem~\ref{thm:main}, we prove exponential stability of the optimal (safely reachable) steady-state, as an extension of~\cite{LIMON18,kohler19nonlinear} to the robust setting.

\subsubsection{New Terminal Ingredients}
The main approach in MPC design for ensuring stability and recursive feasibility is to introduce terminal ingredients, i.e. a terminal cost $V_{\mathrm{f}}$ and a terminal set $\mathcal{X}_{\mathrm{f}}$. 
Determining the setpoint $(x^{\mathrm{s}},v^{\mathrm{s}})$ online and occurring disturbances, further complicate their design.

The proposed approach determines the terminal set size $\alpha$ online, using one additional scalar variable similar to~\cite{kohler19nonlinear}, which is less conservative than the design in~\cite{LIMON18}. 
Furthermore, by parametrizing the terminal set $\mathcal{X}_{\mathrm{f}}$ with the incremental Lyapunov function $V_{\delta}$, we can derive intuitive formulas that ensure robust recursive feasibility in terms of lower and upper bounds on $\alpha$~\eqref{equ:constr_term_set_1}. 
As a result, we improve and extend~\cite{LIMON18,kohler19nonlinear} to the case of nonlinear robust setpoint tracking.
The properties of the terminal ingredients are summarized in Proposition~\ref{thm:set}. 
\begin{proposition}
\label{thm:set}
The set of constraints~\eqref{equ:new_constraints},~\eqref{equ:constr_term_set_1} and~\eqref{equ:constr_term_set_2} together with~\eqref{equ:constr_term_set_3} and the terminal controller $k_{\mathrm{f}} = v^{\mathrm{s}}$, provide a terminal set that ensures the following sufficient properties for robust recursive feasibility 
(c.f.~\cite[Ass.~7]{KoehlerCompEff18}).
\begin{itemize}
    \item The terminal set constraint $x(N|t)\in\mathcal{X}_{\mathrm{f}}(x^{\mathrm{s}},\alpha)$ is robust recursively feasible for fixed values $x^{\mathrm{s}},v^{\mathrm{s}},\alpha$. 
    \item The tightened state and input constraints~\eqref{equ:tightened_const} are satisfied within the terminal region.
\end{itemize}
\end{proposition}
\begin{proof}
The candidate $\sqrt{V_{\delta}(\tilde{x}^+,x^+)}\leq \rho_{\mathrm{d}}^N\overline{w}_{\mathrm{d}}$ (c.f.~\cite[Ass.~7]{KoehlerCompEff18}) satisfies the terminal constraint~\eqref{equ:constr_term_set_2} by using
\begin{align*}
&\sqrt{V_{\delta}(\tilde{x}^+,x^{\mathrm{s}})}\stackrel{\eqref{equ:inc_stab_equ_3},\eqref{eq:V_delta_w_diff}}{\leq} \rho_{\mathrm{d}}\sqrt{V_{\delta}(x,x^{\mathrm{s}})}+\rho_{\mathrm{d}}^N\overline{w}_{\mathrm{d}}\\
\stackrel{\eqref{equ:constr_term_set_3}}{\leq}& \rho_{\mathrm{d}}\alpha-\rho_{\mathrm{d}}\frac{1-\rho_{\mathrm{d}}^N}{1-\rho_{\mathrm{d}}}\overline{w}_{\mathrm{d}}+\rho_{\mathrm{d}}^N\overline{w}_{\mathrm{d}}\stackrel{\eqref{equ:constr_term_set_1}}{\leq}  \alpha-\frac{1-\rho_{\mathrm{d}}^N}{1-\rho_{\mathrm{d}}}\overline{w}_{\mathrm{d}}.
\end{align*}
Satisfaction of the tightened constraints~\eqref{equ:tightened_const} inside the terminal set follows with
\begin{align*}
&g_{j,\kappa}(x,v^{\mathrm{s}})+c_j\frac{1-\rho_{\mathrm{d}}^N}{1-\rho_{\mathrm{d}}}\overline{w}_{\mathrm{d}}\\
\stackrel{\eqref{equ:inc_stab_equ_2}}{\leq}& g_{j,\kappa}(x^{\mathrm{s}},v^{\mathrm{s}})+c_j(\sqrt{V_{\delta}(x,x^{\mathrm{s}})}+\frac{1-\rho_{\mathrm{d}}^N}{1-\rho_{\mathrm{d}}}\overline{w}_{\mathrm{d}})\stackrel{\eqref{equ:constr_term_set_1},\eqref{equ:constr_term_set_3}}{\leq}0.\qedhere
\end{align*}
\end{proof}
In addition to the presented terminal set, we consider some Lipschitz continuous terminal cost $V_{\mathrm{f}}$, which satisfies the following conditions in the terminal set with some $c>0$
\begin{subequations}
\label{equ:terminal_cost}
\begin{align}
\label{equ:terminal_cost_1}
    & V_{\mathrm{f}}(f_{\mathrm{d},\kappa}(x,v^{\mathrm{s}}),x^{\mathrm{s}})-V_{\mathrm{f}}(x,x^{\mathrm{s}})\leq -\|x-x^{\mathrm{s}}\|_Q^2, \\
\label{equ:terminal_cost_2}
    & V_{\mathrm{f}}(x,x^{\mathrm{s}})\leq c\|x-x^{\mathrm{s}}\|^2.
\end{align}
\end{subequations}
For the computation of the terminal cost for nonlinear systems with varying setpoints, we refer to~\cite{LIMON18,kohler19nonlinear}.

\subsubsection{Offline/Online Computations}
The procedure for performing the offline calculations can be found in Algorithm~\ref{alg:rmpc_off_calc}. One approach to compute suitable functions $V_{\delta},\kappa,V_{\mathrm{f}}$ using a quasi-LPV parametrization and linear matrix inequalities (LMIs) is described in~\cite{KoehlerQINF18}. The subsequent online calculations can then be performed according to Algorithm~\ref{alg:rmpc_on_calc}.
\begin{algorithm} 
  \begin{algorithmic}[1]
    \State Determine a stabilizing feedback $\kappa$ and a corresponding incremental Lyapunov function $V_{\delta}$ (Ass.~\ref{ass:loc_inc_stab}).
    \State Compute constant $\overline{w}_{\mathrm{c}}$ satisfying~\eqref{eq:V_delta_w_diff}.
    \State Compute constants $ c_j$ satisfying~\eqref{equ:inc_stab_equ_2}.
    \State Define sampling time $h$, compute $\rho_{\mathrm{d}}$, $\overline{w}_{\mathrm{d}}$ (Sec.~\ref{sec:design_quantities}).
    \State Determine terminal cost $V_{\mathrm{f}}(x,x^{\mathrm{s}})$ satisfying~\eqref{equ:terminal_cost}.
  \end{algorithmic} 
  \caption{Offline calculations for RMPC design.}
  \label{alg:rmpc_off_calc}
\end{algorithm}
\begin{algorithm} 
  \begin{algorithmic}[1] 
    \State Solve the MPC problem from~\eqref{equ:rmpc_opt_problem}.
    \State Apply input $u_t=\pi_{\text{MPC}}(x_{t_k},y_{t_k}^{\mathrm{d}})+\kappa(x_t),t \in [t_k,t_k+h)$.
  \end{algorithmic} 
  \caption{Online calculations, executed at every time step $t_k$, $k\in\mathbb{N}$ during the sampling time interval $[t_k,t_k+h)$.
  }
  \label{alg:rmpc_on_calc}
\end{algorithm}

\subsubsection{Closed-Loop Properties}
In the following, we derive the closed-loop properties of the proposed scheme. 
The set of safely reachable steady-state outputs $y^{\mathrm{s}}$ is given by 
$\mathbb{Y}_{\mathrm{s}}:=\{y^s \in \mathbb{R}^q |~g_{j}(m_x(y_{\mathrm{s}}),m_u(y_{\mathrm{s}}))+c_j\overline{w}_{\mathrm{d}}/(1-\rho_{\mathrm{d}})\leq 0,~j=1,\dots,r\}$. The optimal (safely reachable) setpoint $y_{\mathrm{opt}}^{\mathrm{s}}$, is the minimizer to the steady-state optimization problem $\!\min_{y_{\mathrm{s}}\in\mathbb{Y}_{\mathrm{s}}}\|y^{\mathrm{s}}-y^{\mathrm{d}}\|_{Q_{\mathrm{o}}}^2$.

The following technical condition is sufficient to ensure convergence to the optimal steady-state, compare~\cite{LIMON18}, \cite{kohler19nonlinear}. 
\begin{assumption}
\label{ass:Limon}
There exist (typically unknown) unique functions $m_x,m_u$, that are Lipschitz continuous. Furthermore, the set of safe output references $\mathbb{Y}_{\mathrm{s}}$ is convex.
\end{assumption}
Consequently, save operation and stability convergence is guaranteed due to the following theorem.
\begin{theorem}
\label{thm:main}
Let Assumption~\ref{ass:loc_inc_stab} hold and suppose that Problem~\eqref{equ:rmpc_opt_problem} is feasible at $t=0$. 
Then Problem~\eqref{equ:rmpc_opt_problem} is recursively feasible and the posed constraints $\mathcal{Z}$ are satisfied for the resulting closed loop (Algorithm~\ref{alg:rmpc_on_calc}), i.e., the system operates safely.
Suppose further that Assumption~\ref{ass:Limon} holds and $y^{\mathrm{d}}$ is constant.
Then the optimal (safely reachable) setpoint $x^{\mathrm{s}}_{\mathrm{opt}}$ is practically exponentially stable for the closed-loop system and the output $y$ practically exponentially converges to $y_{\mathrm{opt}}^{\mathrm{s}}$. 
\end{theorem}
\begin{proof}
The safety properties of the proposed scheme are due to the RMPC theory in~\cite{KoehlerCompEff18}, using the known contraction rate $\rho_{\mathrm{d}}$ and the constant $\overline{w}_{\mathrm{d}}$ (bounding the uncertainty) to compute a safe constraint tightening in~\eqref{equ:tightened_const}.
Proposition~\ref{thm:set} ensures that the novel design of the terminal ingredients using~\eqref{equ:constr_term_set_1}, \eqref{equ:constr_term_set_2} and \eqref{equ:constr_term_set_3} also satisfies the conditions in~\cite[Ass.~7]{KoehlerCompEff18} for fixed values $x^{\mathrm{s}},v^{\mathrm{s}},\alpha$.
The stability/convergence properties of the considered formulation are based on the non-empty terminal set ($\alpha>0$) with corresponding terminal cost~\eqref{equ:terminal_cost} and convexity of $\mathbb{Y}_{\mathrm{s}}$ (Ass.~\ref{ass:Limon}), which allow for an incremental change in $y_{\mathrm{s}}$ towards the desired output $y^{\mathrm{d}}$, compare~\cite{LIMON18,kohler19nonlinear} for details. 
Thus, the Lyapunov arguments in~\cite{kohler19nonlinear} remain valid with a quadratically bounded Lyapunov function $V_t$ satisfying $V_{t+1}-V_t\leq -\gamma\|x_t-x_{\mathrm{opt}}^{\mathrm{s}}\|^2+\alpha_{\mathrm{w}}(\overline{w}_{\mathrm{d}})$ with a positive definite function $\alpha_{\mathrm{w}}$~\cite{KoehlerCompEff18}, bounding the effect of the model mismatch. 
This implies practical exponential stability of $x_{\mathrm{opt}}^{\mathrm{s}}$. 
\end{proof}
\begin{remark}
\textit{Practical} stability of $x^{\mathrm{s}}_{\mathrm{opt}}$ implies that the system converges to a neighborhood of the optimal setpoint $y^{\mathrm{s}}_{\mathrm{opt}}$.
\end{remark}
\begin{remark}
\label{rem:non_conv}
Convexity of $\mathbb{Y}_s$ and uniqueness of the functions $m_x$, $m_u$ (Ass.~\ref{ass:Limon}) are strong assumption for general nonlinear problems. In particular, for the considered redundant 7-DOF robotic manipulator (Sec.~\ref{sec:experiment}), the functions $m_x,m_u$ are not unique (kinematic redundancy) and the feasible steady-state manifold $\mathbb{Y}_{\mathrm{s}}$ is not convex (collision avoidance constraint).  
Nevertheless, the safety properties are not affected by Ass.~\ref{ass:Limon} and in the experimental implementation, the RMPC typically converges to some (not necessarily unique) steady-state.
\end{remark}

\subsection{Approximate MPC}
\label{sec:Method_AMPC}
In the following, we introduce the AMPC, which provides an explicit approximation $\pi_{\text{approx}}$ of the RMPC control law $\pi_{\text{MPC}}$, yielding a significant decrease in computational complexity. 
In particular, as demonstrated in the numerical study in~\cite[Sec.~9.4]{chen2019large}, approximate MPC without additional modifications will in general not satisfy the constraints. Consequently, the core idea of the presented AMPC approach is to compensate for inaccuracies of the approximation by introducing additional robustness within the RMPC design.
In the following, we present a solution to obtain statistical guarantees (Sec.~\ref{sec:statistical_guarantees}) for the application of the resulting AMPC. To that end, we introduce an improved validation criterion (Prop.~\ref{prop:ampc_stab}, Sec.~\ref{sec:val_criterion}) compared to the one in~\cite{hertneck18}, being more suitable for real world applications.

\subsubsection{Validation Criterion}
\label{sec:val_criterion}

The following proposition provides a sufficient condition for AMPC safety guarantees.
\begin{proposition}
\label{prop:ampc_stab}
Let Assumption~\ref{ass:loc_inc_stab} hold.
Suppose the model mismatch between the real and the nominal system satisfies
\begin{align}
\label{equ:ampc_stab_equ1}
& \sqrt{V_\delta(f_{\mathrm{d},\kappa}(x,v),f_{\mathrm{d},\kappa}(x,v)+d_{\mathrm{w},\mathrm{d}})} \leq \overline{w}_{\mathrm{d},\mathrm{model}}, 
\end{align}
$ \forall (x,v+\kappa(x))\in\mathcal{Z}$, $\forall d_{w,d}\in\mathcal{W}_{\mathrm{d}}(x,v)$. If $\pi_{\text{MPC}}$ is designed with some $\overline{w}_{\mathrm{d}}$ and the approximation $\pi_{\text{approx}}$ satisfies
\begin{equation}
\!\begin{aligned}
\label{equ:ass_within_prop_2}
& \sqrt{V_\delta(f_{\mathrm{d},\kappa}(x,\pi_{\text{approx}}(x)), f_{\mathrm{d},\kappa}(x,\pi_{\text{MPC}}(x))} \\
& \quad \leq \overline{w}_{\mathrm{d},\text{approx}}:=\overline{w}_{\mathrm{d}}-\overline{w}_{\mathrm{d},\mathrm{model}},
\end{aligned}
\end{equation}
for any state $x$ with \eqref{equ:rmpc_opt_problem} being feasible, 
then the AMPC ensures the same properties as the RMPC in Theorem~\ref{thm:main}.
\end{proposition}
\begin{proof}
We use the following bound on the perturbed AMPC:
\begin{align*}
& \sqrt{V_\delta(f_{\mathrm{d}}(x,\pi_{\text{approx}})+d_{\mathrm{w},\mathrm{d}}, f_\mathrm{d}(x,\pi_{\text{MPC}}))}\\
& \stackrel{\text{\eqref{equ:norm_like_cond}}}{\leq} \sqrt{V_\delta(f_{d}(x,\pi_{\text{approx}})+d_{w,d}, f_{d}(x,\pi_{\text{approx}}))} \\ 
& + \sqrt{V_\delta(f_{d}(x,\pi_{\text{approx}}), f_d(x,\pi_{\text{MPC}}))} \\
& \stackrel{\eqref{equ:ampc_stab_equ1},\eqref{equ:ass_within_prop_2}}{\leq} \bar{w}_{\mathrm{d},\mathrm{model}} + \bar{w}_{\mathrm{d},\text{approx}} \stackrel{\eqref{equ:ass_within_prop_2}}{=} \bar{w}_{\mathrm{d}}.
\end{align*}
Then, the properties follow from Theorem~\ref{thm:main}.\qedhere
\end{proof}

\subsubsection{Statistical Guarantees}
\label{sec:statistical_guarantees}
In practice, guaranteeing a specified error $\overline{w}_{\mathrm{d},\text{approx}}$ for all possible values $(x,y^{\mathrm{d}})$ with a supervised learning approach is difficult, especially for deep NNs.
However, it is possible to make statistical statements about $\pi_{\text{approx}}$ using Hoeffding's inequality~\cite{hoeffding63}.
For the statistical guarantees, we adopt the approach from~\cite{hertneck18} and use our improved validation criterion as introduced in Proposition~\ref{prop:ampc_stab}. 
\begin{assumption}
\label{ass:deterministic}
The prestabilized, disturbed system dynamics $f_{\mathrm{d},\mathrm{w},\kappa}$ characterize a deterministic (possibly unknown) map. 
\end{assumption}
We validate full trajectories under the AMPC with independent and identically distributed (i.i.d.) initial condition and setpoints. Due to Assumption~\ref{ass:deterministic}, also the trajectories themselves are i.i.d.. Specifically, we define a trajectory as
\begin{equation}
\label{equ:X_i}
\!\begin{aligned}
& X_i := \{ x(k), k \in \mathbb{N}: x_i(0)~\text{feasible at}~t=0, \\
& \qquad x(k+1) = f_{\mathrm{d},\mathrm{w},\kappa}(x(k),\pi_{\text{approx}}(x(k),y_i^d)) \}. 
\end{aligned}
\end{equation}
Further, we consider the indicator function based on~\eqref{equ:ass_within_prop_2}
\begin{equation}
\nonumber
\label{equ:hertneck_indic_func}
  I(X_i)= 
\begin{cases}
    1,& \text{if } \sqrt{V_\delta(f_{\mathrm{d},\kappa}(x,\pi_{\text{approx}}), f_{\mathrm{d},\kappa}(x,\pi_{\text{MPC}}))} \\
    & \quad \leq \overline{w}_{\mathrm{d},\text{approx}},~\forall x \in X_i \\
    0,              & \text{otherwise}.
\end{cases}
\end{equation}
The indicator measures, whether for any time step along the trajectory, there is a discrepancy larger than $\overline{w}_{\mathrm{d},\text{approx}}$ between the ideal trajectory with $\pi_{\text{MPC}}$ and the trajectory with the approximated input $\pi_{\text{approx}}$.
The empirical risk is given as $\tilde{\mu} = \frac{1}{b} \sum_{j=1}^{b} I(X_j)$ for $b$ sampled trajectories, while $\mu$ is denoting the true expected value of the random variable. With Hoeffding's inequality the following Lemma can be derived.
\begin{lemma}
\label{lem:stat_val}
\cite[Lemma~1]{hertneck18}
Suppose Assumption~\ref{ass:deterministic} holds. Then the condition $\mathbb{P}\left[I(X_i) = 1 \right] \geq \mu_{\mathrm{crit}} := \tilde{\mu} - \sqrt{-\ln(\delta_h/2)/(2b)}$, holds at least with confidence $1-\delta_h$.
\end{lemma}
\begin{remark}
In practice, it is not possible to check for infinite length trajectories $X_i$. Since in our definition, the reference $y_i^{\mathrm{d}}$ is fixed along the whole trajectory, we do the validation until a steady state is reached below a certain threshold.
\end{remark}
We provide the following illustration: given a large enough number of successfully validated trajectories, we obtain a high empirical risk, e.g. $\tilde{\mu} \approx 99\%$. This result ensures that with confidence of e.g. $(1-\delta_h)\approx 99\%$,~\eqref{equ:ass_within_prop_2} holds at least with probability $\mu_{\mathrm{crit}}$ (e.g. $\mu_{\mathrm{crit}} \approx 98$\%) for a new trajectory with initial condition $(x,y^{\mathrm{d}})$.
Thus, with high probability, the guarantees in Proposition~\ref{prop:ampc_stab} (safety and stability) hold.

\subsubsection{Algorithm}
The overall procedure for the AMPC is summarized in Algorithm~\ref{alg:ampc}, based on Hertneck et al. in~\cite{hertneck18}.
\begin{algorithm} 
  \begin{algorithmic}[1] 
    \State Choose $\overline{w}_{\mathrm{d}}$, determine $\overline{w}_{\mathrm{d},\mathrm{model}}$ and calculate $\overline{w}_{\mathrm{d},\text{approx}}$. 
    \State Design the RMPC according to Algorithm~\ref{alg:rmpc_off_calc}.
    \State Learn $\pi_{\text{approx}} \approx \pi_{\text{MPC}}$. 
    \State Validate $\pi_{\text{approx}}$ according to Lemma~\ref{lem:stat_val}.
    \State If the validation fails, repeat the learning from step 3.
  \end{algorithmic} 
  \caption{Procedure for the AMPC.}
  \label{alg:ampc}
\end{algorithm}


\section{Robot Experiments}
\label{sec:experiment}
We demonstrate the proposed RMPC and AMPC approaches on a KUKA LBR4+ robotic manipulator (Fig.~\ref{fig:robot_experiment}).

\subsection{Robotic System}
Several works investigated the dynamics formulation of the KUKA LBR4+ and LBR iiwa robotic manipulators~\cite{Jubien14,Sturz17}, with dynamic equations of the form
$M(q)\ddot{q}+b(q,\dot{q}) = \tau$.
Here, $\tau$ denotes the applied torque and $q,\dot{q},\ddot{q}$ the joint angle, joint velocity and joint acceleration~\cite{Siciliano2008}.

\subsubsection{System Formulation}

In this work, we leverage existing low-level controllers as an inverse dynamics inner-loop feedback linearization ending up with a kinematic model that assumes direct control of joint accelerations, i.e., $\ddot{q}_t = u_t$.
Such a description is not uncommon for designing higher-level controllers in robotics, compare e.g. the MPC scheme in~\cite{carron2019data} based on a kinematic model.
As the control objective, we aim for tracking a given reference $y^d$ in the task space with the manipulator end effector position, defined as $y$. Since this position only depends on the first four joints, we  consider those for our control design.
The resulting nonlinear system with state $x_t= [ q_t,\dot{q}_t ]^\top$ is given by
\begin{align}
    \dot{x}_t=[\dot{q}_t^\top ,u_t^\top]^\top,~y_t=o(x_t,u_t),~x_t\in\mathbb{R}^8,~u_t\in\mathbb{R}^4.
\end{align}
The output $y=o(x,u)$ is given by the forward kinematic:
\begin{equation}
\label{equ:apollo_endeff_pos_2}
\begin{gathered}
y = o(x,u) = \\
\begin{pmatrix}
  \scriptstyle -C_1c_{q_1}s_{q_2} - C_2s_{q_4}(s_{q_1}s_{q_3} -c_{q_1}c_{q_2}c_{q_3}) - C_2c_{q_1}c_{q_4}s_{q_2} \\
  \scriptstyle C_2s_{q_4}(c_{q_1}s_{q_3}+c_{q_2}c_{q_3}s_{q_1}) -C_1s_{q_1}s_{q_2} - C_2c_{q_4}s_{q_1}s_{q_2} \\
  \scriptstyle C_1c_{q_2} + C_2c_{q_2}c_{q_4} +C_2c_{q_3}s_{q_2}s_{q_4} + C_3 
\end{pmatrix},
\end{gathered}
\end{equation}
where $s_{q_i}$ and $c_{q_i}$ denote the sine and cosine of $q_i$, respectively, and $C_1=0.4, C_2=0.578, C_3=0.31$.

\subsubsection{Constraints}
States and inputs are subject to the following polytopic constraints: joint angles $q_i$ can turn less than $\pm180^\circ$ (exact values can be found in~\cite{nubert19}), joint velocity $|\dot{q}|\leq 2.3\frac{\text{rad}}{\text{s}}$, and joint acceleration $|\ddot{q}|\leq 8\frac{\text{rad}}{\text{s}^2}$.

More interestingly, we also impose constraints on the output function $y$ to ensure obstacle avoidance in the Cartesian space. We approximate the obstacles with differentiable functions, compare Figure~\ref{fig:output_constraints}. This allows for a simpler implementation and design.
\begin{figure}[t]
	\centering
	\includegraphics[width=0.97\linewidth]{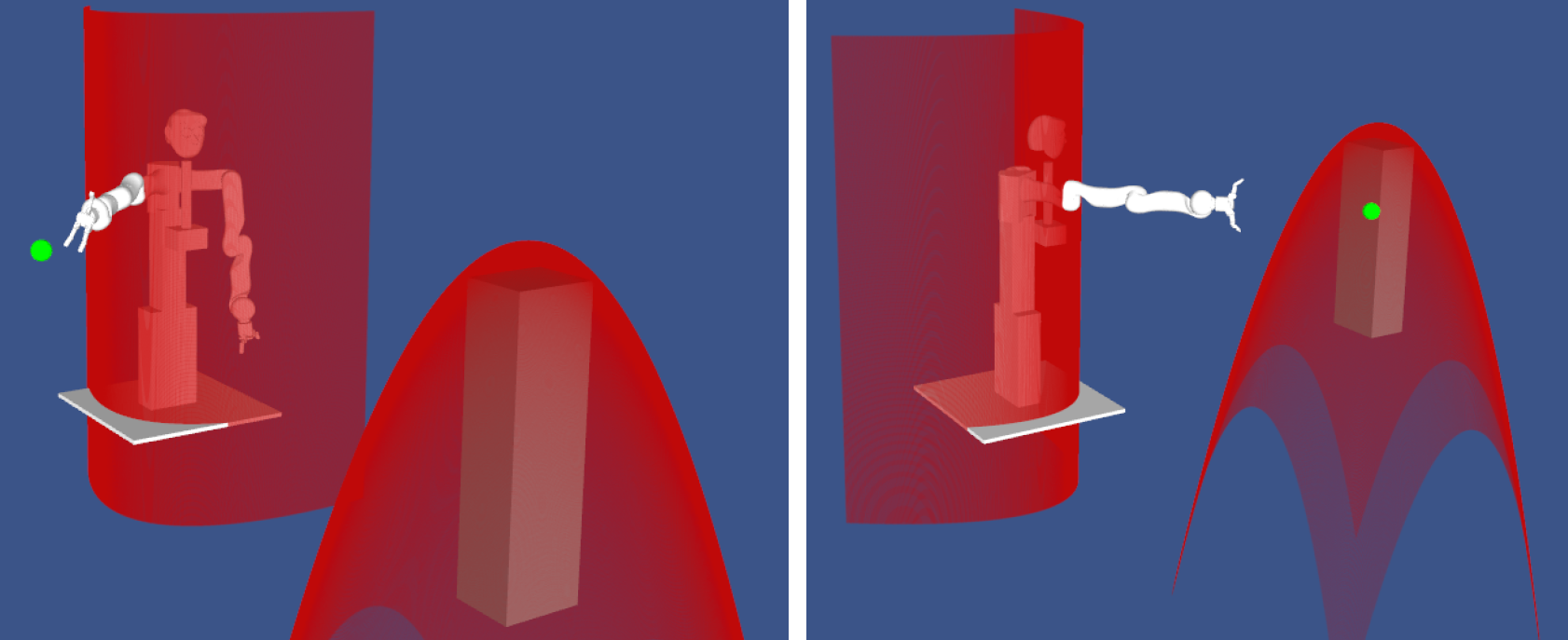}
	\caption{Visualization of the output constraints. We use (quadratic) differentiable functions to over-approximate the non-differentiable obstacles.}
	\label{fig:output_constraints}
\end{figure}
For example,
\begin{equation}
\label{equ:obst_vert_outp_constr}
g_p(x,u) = -(y_1 - y_1^{\mathrm{o}}) -C \left((y_3 - y_3^{\mathrm{o}})^2 + (y_2 - y_2^{\mathrm{o}})^2\right)\leq 0, \nonumber
\end{equation}
models the box-shaped obstacle, with the obstacle position $y^{\mathrm{o}}=(y_1^{\mathrm{o}},y_2^{\mathrm{o}},y_3^{\mathrm{o}})$, the end effector $y=(y_1,y_2,y_3)$, and here $C=2$.
Similarly, we introduce a nonlinear constraint that prevents the robot from hitting itself (see Figure~\ref{fig:output_constraints}).
\begin{remark}
This constraint formulation uses a simple (conservative) over-approximation and assumes static obstacles. 
Both limitations can be addressed by using the exact reformulation in~\cite{Zhang2017} based on duality and using the robust extension in~\cite{Soloperto18} for uncertain moving obstacles.
\end{remark}

\subsection{Robust MPC Design}
In general, the dynamic compensation introduced in the previous subsection is not exact and hence, the resulting model mismatch needs to be addressed in the robust design.

\subsubsection{Determination of Disturbance Level}
\label{sec:dist_det}
For the determination of the disturbance level, we sample trajectories for a specified sampling time and compare the observed trajectory to the nominal prediction for each discrete-time step. The deviation of the two determines the disturbance bound introduced for the discrete-time case, i.e. $d_{\mathrm{w},\mathrm{d}}$. In Figure~\ref{fig:obs_dist}, a plot of the $\infty$-norm of the observed disturbance with respect to the applied acceleration is shown. 
\begin{figure}[t]
	\centering
	\includegraphics[width=0.8\linewidth]{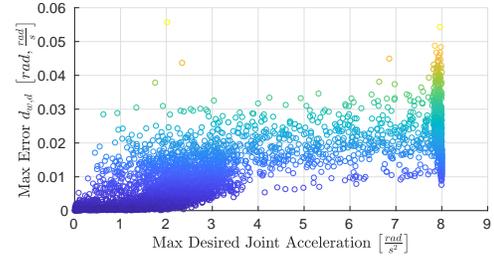}
	\caption{Observed disturbance with respect to the applied acceleration for a sampling rate of $2.5\,\text{Hz}$. Proportionality-like behavior is apparent.}
	\label{fig:obs_dist}
\end{figure}
The maximal observed model mismatch satisfies $\|d_{\mathrm{w},\mathrm{d}}\|_{\infty}\leq 0.06$. 
As a precaution we add some tolerance and use  $\|d_{\mathrm{w},\mathrm{d}}\|_\infty \leq 0.1$ for our design.
From the figure, it can be seen that the induced disturbance can be larger for higher accelerations. This behavior is not surprising, since the low level controllers have more difficulties to follow the reference acceleration for more dynamic movements. 
Using $\overline{w}_{\mathrm{c}}(x,u)=c_0+c_1\|u\|_{\infty}+c_2\|\dot{q}\|_\infty$ instead of a constant bound, could help to further decrease conservatism (compare~\cite{KoehlerCompEff18}). Furthermore, the uncertainty could also be reduced by improving the kinematic model using data, as e.g. done in~\cite{carron2019data} with an additional gaussian process (GP) error model.

\subsubsection{Computations}
\label{sec:rmpc_computations}
The offline computations are done according to Algorithm~\ref{alg:rmpc_off_calc}. We consider a quadratic incremental Lyapunov function $V_{\delta}(x,z)=\|x-z\|_{P_\delta}^2$ and a linear feedback $\kappa(x)=K_{\delta}x$, both computed using tailored LMIs (incorporating~\eqref{eq:V_delta_w_diff},~\eqref{equ:inc_stab_equ_2}), compare~\cite{nubert19}.
The terminal cost $V_{\mathrm{f}}$ is given by the LQR infinite horizon cost.
The online computations from Algorithm~\ref{alg:rmpc_on_calc} are performed in a real-time C++ environment by deploying the CasADi C++ API for solving the involved optimization problem~\cite{Andersson2018}.
The feedback $\kappa(x_t)$ is updated with a rate of $1\,\text{kHz}$ -- hence, it can be considered as being continuous-time for all practical purposes. Furthermore, $\pi_{\text{MPC}}$ is evaluated every $h=400 \,\text{ms}$.

\subsection{Experimental Results RMPC}
With the RMPC design, we demonstrate a reliable and safe way for controlling the end effector position of the robotic manipulator.  
An exemplary trajectory on the real system can be observed in Figure~\ref{fig:robot_experiment}, where the end effector tracks the reference, which is set by the user. 
\begin{figure}[hbt!]
	\begin{subfigure}{\linewidth}
	\centering
	\includegraphics[width=.9\linewidth]{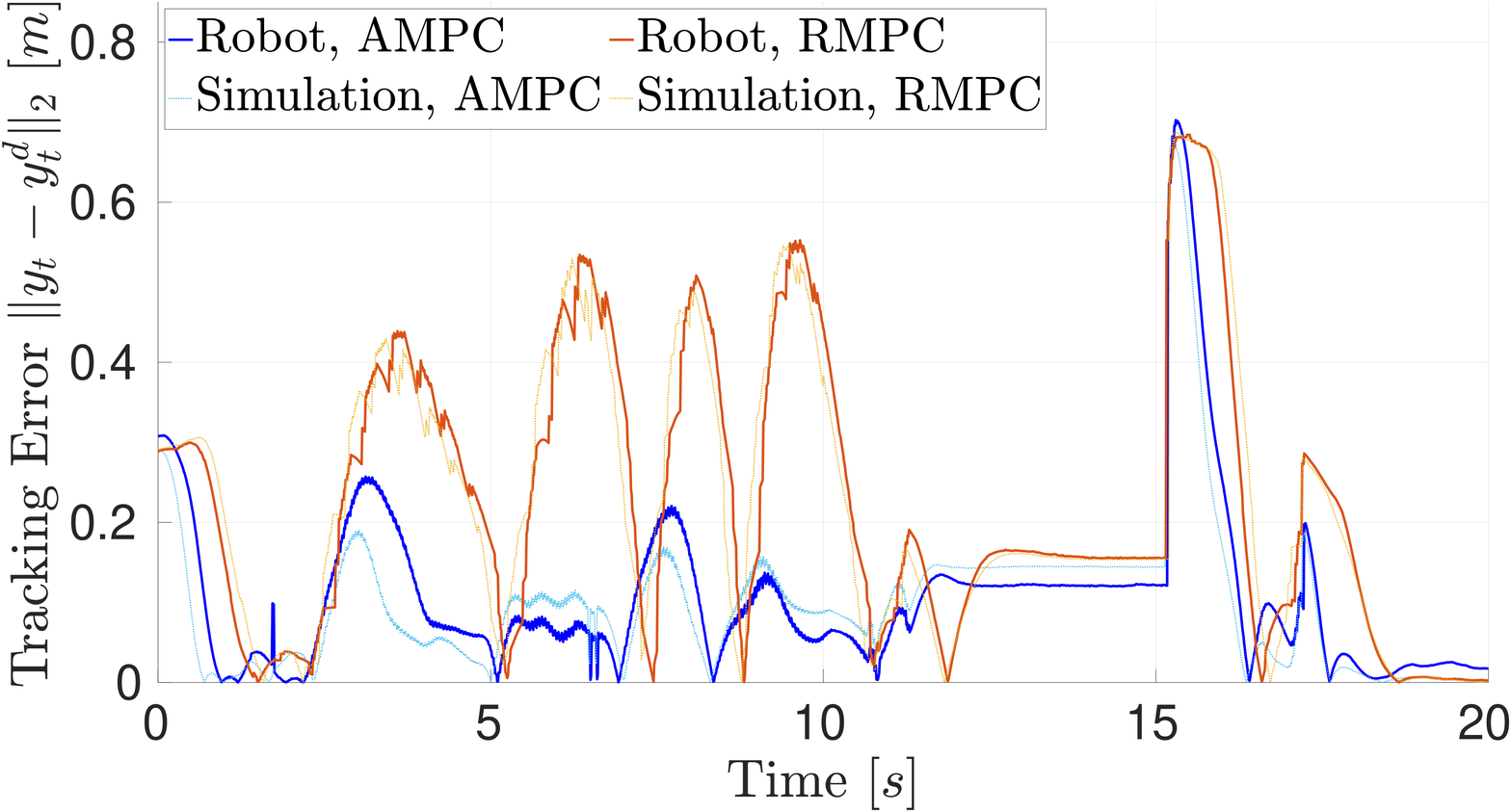}
	\label{fig:track_err_and_cont_inp_a}
	\end{subfigure}
	\begin{subfigure}{\linewidth}
	\centering
	\includegraphics[width=.9\linewidth]{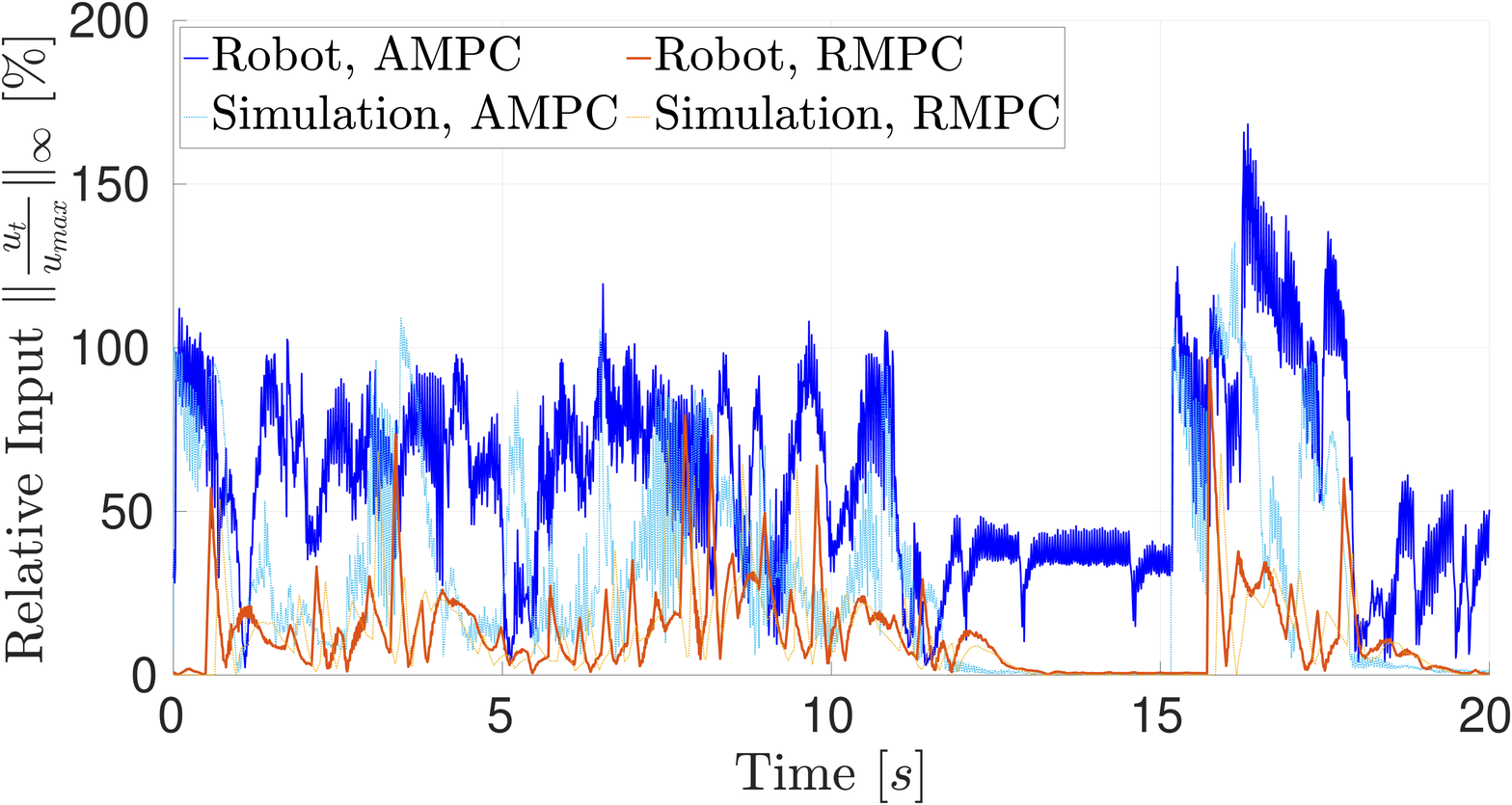}
	\label{fig:track_err_and_cont_inp_b}
	\end{subfigure}
    \caption{
	Experimental (solid) and simulation (dashed) data of RMPC (blue-colored) and AMPC (orange-colored) with the same reference $y_t^{\mathrm{d}}$. Reference $y^\mathrm{d}$ is continuously moving for $t \in [0,11]$s; constant, but unreachable in the interval $t\in[11,15]$s; and moving again after a step for $t \in [15,19]$s; with $u_t=\kappa(x_t)+\pi_{\text{MPC}}(x_{t_k})$ and $u_t=\kappa(x_t)+\pi_{\text{approx}}(x_{t_k})$.
	}
	\label{fig:track_err_and_cont_inp}
\end{figure}
Even though the direct way is obstructed by an obstacle, the controller obtains a solution while keeping safe distance to it. Figure~\ref{fig:track_err_and_cont_inp} shows the tracking error and closed-loop input of the RMPC for an exemplary use-case. The controller is able to track the reference. However, due to the computational complexity and its induced delay, the controller has a larger tracking error in intervals of changing set points (interval $[2,11]$s in Fig.~\ref{fig:track_err_and_cont_inp}).
Note that the constraint tightening of the considered RMPC method only restricts future control actions and thus the scheme can in principle utilize the full input magnitude.
However, due to the  combination of the velocity constraint and the long sampling time $h$, the full input is only utilized by the AMPC, with the faster sampling time.
More experimental results can be observed in the supplementary video\footnote{\url{https://youtu.be/c5EekdSl9To}}.

Integrating the tracking control within a single optimization problem and automatically resolving corner cases such as unreachable setpoints are particular features that make the deployment of the approach simple, safe, and reliable in practice. As expected by the considered robust design, in thousands of runs (one run corresponds to one initial condition and one output reference), the robot never came close to hitting any of the obstacles (e.g. video: $\text{1:15}\,\text{min}$). 
This is the result of using the bound $\overline{w}_{\mathrm{d}}$ on the model mismatch, implying safe but conservative operation.
Furthermore, the controller is able to steer the end effector along interesting trajectories in order to avoid collisions (e.g. video: $\text{1:50}\,\text{min}$). 

\subsection{AMPC Design}
For the robot control, the AMPC is  designed according to Algorithm~3.
For this purpose, we first design an RMPC with a sampling time of $h=40\,\text{ms}$, i.e., ten times faster than the previous RMPC.
To simplify the learning problem, we only consider the self-collision avoidance constraint.
Therefore, the MPC control law $\pi_{\text{MPC}}$ depends on the state $x \in\mathbb{R}^8$ and the desired reference $y^{\mathrm{d}} \in\mathbb{R}^3$, i.e., on $11$ parameters in total. 

To obtain the necessary precision for the AMPC, interesting questions emerged regarding the structure of the used NN, its training procedure and the sampling of the (ground truth) RMPC. Regarding the depth of the network, our observations confirm insights in~\cite{karg18}: deep NNs are better suited to obtain an explicit policy representation.

A tradeoff exists between the higher expressiveness and the slower training of deeper networks. We decided to use a fully connected NN with 20 hidden layers, consecutively shrinking the layers from $1024$ neurons in the first hidden layer to $4$ in the output layer. This results in roughly $5 \cdot 10^6$ trainable parameters in total. All hidden neurons are \emph{ReLu}-activated, whereas the output layer is activated linearly. Other techniques such as batch normalization, regularization, or skip connections did not help to improve the approximation.

The RMPC control law $\pi_{\text{MPC}}(x,y^{\mathrm{d}})$ can become relatively large in magnitude, which makes the regression more difficult. We circumvent this problem by directly learning the applied input $u^*(x,y^{\mathrm{d}})=\pi_{\text{MPC}}(x,y^{\mathrm{d}})+\kappa(x)\in\left[-8\frac{\text{rad}}{\text{s}^2},8\frac{\text{rad}}{\text{s}^2}\right]$. This can be seen as a zero-centered normalization of the reference output, facilitating smaller approximation errors. In addition, $\pi_{\text{approx}}$ can be readily evaluated online, since $\kappa$ is known.

For the training, we use a set of approximately $50 \cdot 10^6$ datapoints which are obtained by offline sampling the RMPC.
Our training corpus consists of a combination of random sampling $\{\left(x^{(j)},y^{\mathrm{d}(j)}\right),\pi_{\text{MPC}}\left(x^{(j)},y^{\mathrm{d}(j)}\right)\}$ and trajectory-based sampling of i.i.d. trajectories $\{\left(x^{j} \in X_i,y^{\mathrm{d},(i)} \right),\pi_{\text{MPC}}\left(x^{(i)},y^{\mathrm{d}(j)}\right) \}$, with trajectory-wise random initial condition $x^{(i)}$ and reference $y^{\mathrm{d}(i)}$.
The former helps the network to get an idea of all areas, whereas the latter one represents the areas of high interest.

Given the AMPC design, we next aim to perform the validation as per Sec.~\ref{sec:Method_AMPC}. 
We execute the validation in simulation, which is deterministic. 
We account for the model mismatch with a separate term during the validation (cf.
Prop.~\ref{prop:ampc_stab}).
We found that for the considered system and controller tasks, performing the validation is demanding. 
Currently, we are able to satisfy criterion~\eqref{equ:ass_within_prop_2}
for approximately $90$\% of all sampled points.  While this is not fully satisfactory for a high-probability guarantees on full trajectories, it is still helpful to understand the quality of the learned controller.
While no failure cases were observed in the experiments reported herein, performing such \emph{a priori} validation for the robot implementation is subject to future work.

\subsection{Experimental Results AMPC}
With the AMPC design, the evaluation time compared to the RMPC is reduced by a factor of $200$ to $1\,\text{ms}$. 
We implemented the AMPC with a sampling time of $40\,\text{ms}$ and due to the short evaluation of less than $1\,\text{ms}$, the control input can be applied immediately for the current sampling interval instead of performing the optimization for the predicted next state. This results in a response time in the interval $[1,40]\,\text{ms}$ for the AMPC instead of $[400,800]\,\text{ms}$ for the RMPC. 
The resulting, more aggressive input can be observed in Figure~\ref{fig:track_err_and_cont_inp}.
Note that the AMPC sometimes violates the input constraints in the shown experiment. 
This is mainly due to a combination of a large control gain in the pre-stabilization $\kappa$ and large measurement noise in the experiment. 
To circumvent this problem, the noise could be considered in the design or a less aggressive feedback $\kappa$ could be used.

We emphasize that the results and achieved performance are significant, considering the $11$ parameters in the nonlinear
MPC, while standard explicit MPC approaches are only applicable to small-medium scale linear problems.


\section{Conclusion}
The approach developed in this letter achieves safe and fast tracking control on complex systems such as modern robots by combining robust MPC and NN control.

The proposed robust MPC ensures safe operation (stability, constraint satisfaction) despite uncertain system descriptions.  What is more, the MPC scheme simplifies complex tracking control tasks to a single design step by joining otherwise often separate planning and control layers: real-time control commands are directly computed for given reference and constraints.
Our experiments on a KUKA LBR4+ arm are the first to demonstrate such robust MPC on a real robotic system. The proposed RMPC thus, provides a complete framework for tracking control of complex robotic tasks.

We tackled the computational complexity of MPC in fast robotics applications by proposing an approximate MPC. This approach replaces the online optimization with the evaluation of a NN, which is trained and validated in an offline fashion on a suitably defined robust MPC. The proposed approach demonstrates significant speed and performance improvements. 
Again, the presented experiments are the first to demonstrate the suitability of such NN-based control on real robots.
Providing \emph{a priori} statistical guarantees for such robot experiments by further improving the learning and validation procedures are relevant topics for future work.


\section{Acknowledgments}
The authors thank A. Marco and F. Solowjow for helpful discussions, and their colleagues at MPI-IS who contributed to the Apollo robot platform.


\bibliographystyle{IEEEtran}
\bibliography{icra_ral_2020}


\end{document}